\documentclass{article}

\usepackage[final,nonatbib]{nips_2016}

\usepackage[utf8]{inputenc} 
\usepackage[T1]{fontenc}    
\usepackage{hyperref}       
\usepackage{url}            
\usepackage{booktabs}       
\usepackage{amsfonts}       
\usepackage{amsmath}
\usepackage{amssymb}
\usepackage{nicefrac}       
\usepackage{microtype}      
\usepackage{graphicx}

\newcommand{\RR}{\mathbb{R}}
\newcommand{\NN}{\mathbb{N}}

\newcommand{\B}[1]{\mathbf{#1}}
\newcommand{\lgl}{<_{gl}}
\newcommand{\leqgl}{\leq_{gl}}

\usepackage{ulem}
\usepackage{tikz}

\newtheorem{theorem}{Theorem}
\newtheorem{lemma}{Lemma}

\newtheorem{definition}{Definition}
\newtheorem{assumption}{Assumption}

\newenvironment{proof}[1][]{\noindent {\bf Proof #1:\;}}{\hfill $\Box$}

\title{Sorting out typicality with the inverse moment matrix SOS polynomial}

\author{
  Jean-Bernard Lasserre \\
  LAAS-CNRS \& IMT\\
	Universit\'e de Toulouse\\
  31400 Toulouse, France \\
  \texttt{lasserre@laas.fr} 
  \And
  Edouard Pauwels\\
  IRIT \& IMT\\
  Universit\'e Toulouse 3 Paul Sabatier\\
  31400 Toulouse, France \\
  \texttt{edouard.pauwels@irit.fr} 
}

\begin{document}

\maketitle

\begin{abstract}
	We study a surprising phenomenon related to the representation of a cloud of data points using polynomials. We start with the previously unnoticed empirical observation that, given a collection (a cloud) of data points, the sublevel sets of a certain distinguished polynomial capture the shape of the cloud very accurately. This distinguished polynomial is a sum-of-squares (SOS) derived in a simple manner from the inverse of the empirical moment matrix. In fact, this SOS polynomial 
	is directly related to orthogonal polynomials and the \textit{Christoffel} function. This allows to generalize and interpret extremality properties  of orthogonal polynomials and to provide a mathematical rationale for the observed phenomenon. Among diverse potential applications, we illustrate the relevance of our results on a network intrusion detection task for which we obtain performances similar to existing dedicated methods reported in the literature.
\end{abstract}

\section{Introduction}
Capturing and summarizing the global shape of a cloud of points is at the heart of many data processing applications such as novelty detection, outlier detection as well as related unsupervised learning tasks such as clustering and density estimation. One of the main difficulties is to account for potentially complicated shapes in multidimensional spaces, or equivalently to account for non standard dependence relations between variables. Such relations become critical in applications, for example in fraud detection
where a fraudulent action may be the dishonest combination of several actions, each of them being reasonable when considered on their own.

Accounting for complicated shapes is also related to computational geometry and nonlinear algebra applications, for example integral computation \cite{lasserre2015level} and reconstruction of sets from moments data \cite{milanfar,GPSS,lasserre2015algebraic}. Some of these problems have connections and potential applications in machine learning. The work presented in this paper brings together ideas from both disciplines, leading to a method which allows to encode in a simple manner
the global shape and spatial concentration of points within a cloud.

We start with a surprising (and apparently unnoticed) empirical observation. Given a collection of points, one may build up
a distinguished sum-of-squares (SOS) polynomial whose coefficients (or Gram matrix) is the inverse of the empirical moment matrix (see Section \ref{sec:empirical}).
Its degree depends on how many moments are considered, a choice left to the user. Remarkably its sublevel sets capture much of the global shape of the cloud as illustrated in Figure \ref{fig:examples}. This phenomenon is {\it not} incidental as illustrated in many additional examples in Appendix A. To the best of our knowledge, this observation has remained unnoticed 
and the purpose of this paper is  to report this empirical finding to the machine learning community and provide first elements toward a mathematical understanding as well as potential machine learning applications.

The proposed method is based on the computation of the coefficients of a very specific polynomial which depends solely on the empirical moments
associated with the data points. From a practical perspective, this can be done via a single pass through the data, or even in an {\it online} fashion via a sequence of efficient Woodbury updates. 
Furthermore the computational cost of evaluating the polynomial does {\it not} depend on the number of data points which is a crucial difference with existing nonparametric methods such as nearest neighbors or kernel based methods \cite{chandola2009anomaly}. On the other hand, this computation requires the inversion of a matrix 
whose size depends on the dimension of the problem (see Section \ref{sec:empirical}). Therefore, the proposed framework is suited for moderate dimensions and potentially very large number of observations.

In Section \ref{sec:interpretation} we first describe an affine invariance result which suggests that the distinguished SOS polynomial captures very intrinsic properties of clouds of points. In a second step, we provide a mathematical interpretation that supports our empirical findings based on connections with orthogonal polynomials \cite{dunkl2001orthogonal}. We propose a generalization of a well known extremality result for orthogonal univariate polynomials on the real line (or the complex plane) \cite[Theorem 3.1.2]{szego1974orthogonal}. As a consequence, the distinguished SOS polynomial of interest in this paper is understood as the unique optimal solution of a convex optimization problem: minimizing an average value over a structured set of positive polynomials. In addition, we revisit \cite[Theorem 3.5.6]{szego1974orthogonal} about the \textit{Christoffel} function.
The mathematics behind provide a simple and intuitive explanation for the phenomenon that we empirically observed.

Finally, in Section \ref{sec:experiments} we perform numerical experiments on KDD cup network intrusion dataset \cite{lichmanUCI2013}. Evaluation of the distinguished SOS polynomial provides a score that we use as a measure of outlyingness to detect network intrusions (assuming that they correspond to outlier observations). We refer the reader to \cite{chandola2009anomaly} for a discussion of available methods for this task. For the sake of a fair comparison
we have reproduced the experiments performed in \cite{williams2002comparative} for the same dataset. We report results similar to (and sometimes better than) those described in \cite{williams2002comparative} which suggests that the method is comparable to other dedicated approaches for network intrusion detection, including robust estimation and Mahalanobis distance \cite{hadi1994modification,knorr2001robust}, mixture models \cite{olivier1996} and recurrent neural networks \cite{williams2002comparative}.

\begin{figure}[t]
	\centering
	\includegraphics[width=.44\textwidth]{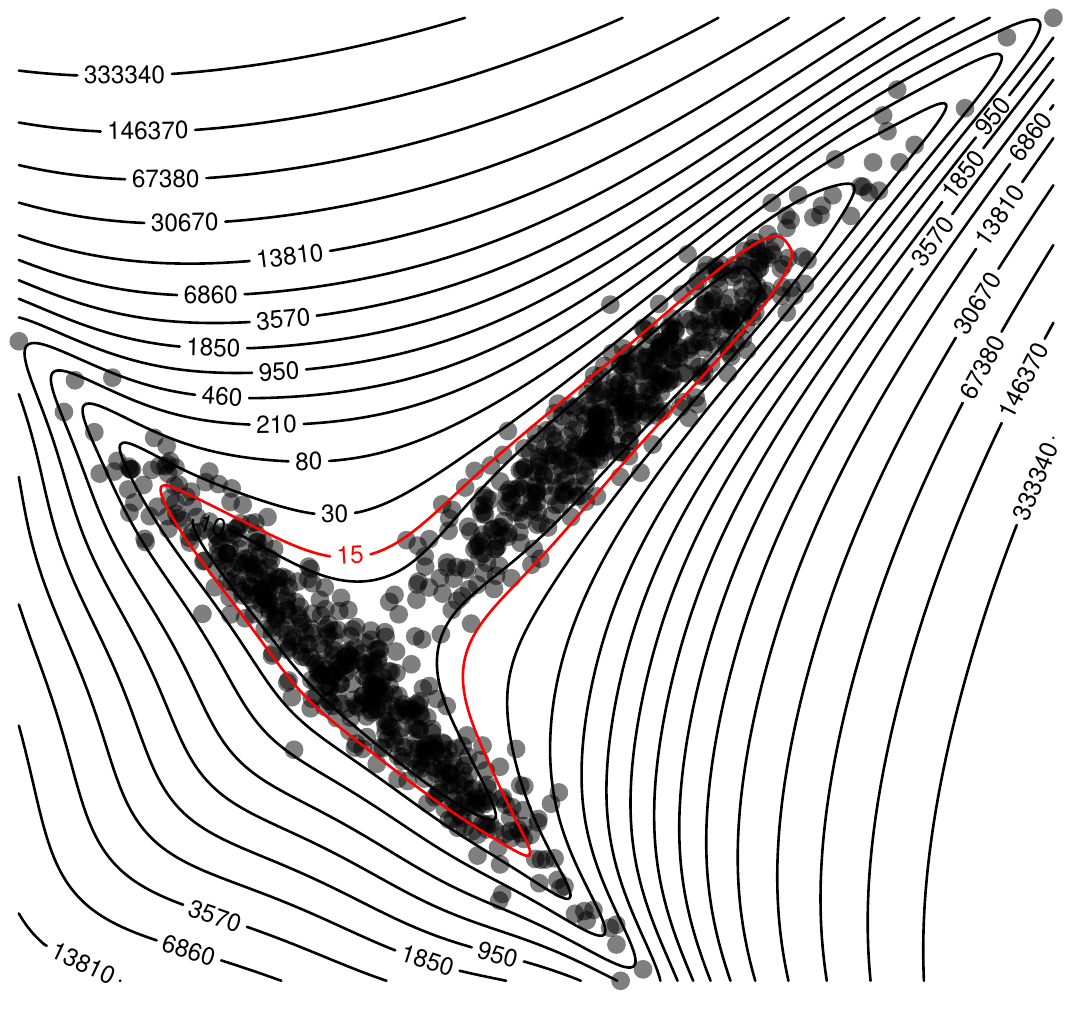}\quad\quad
	\includegraphics[width=.44\textwidth]{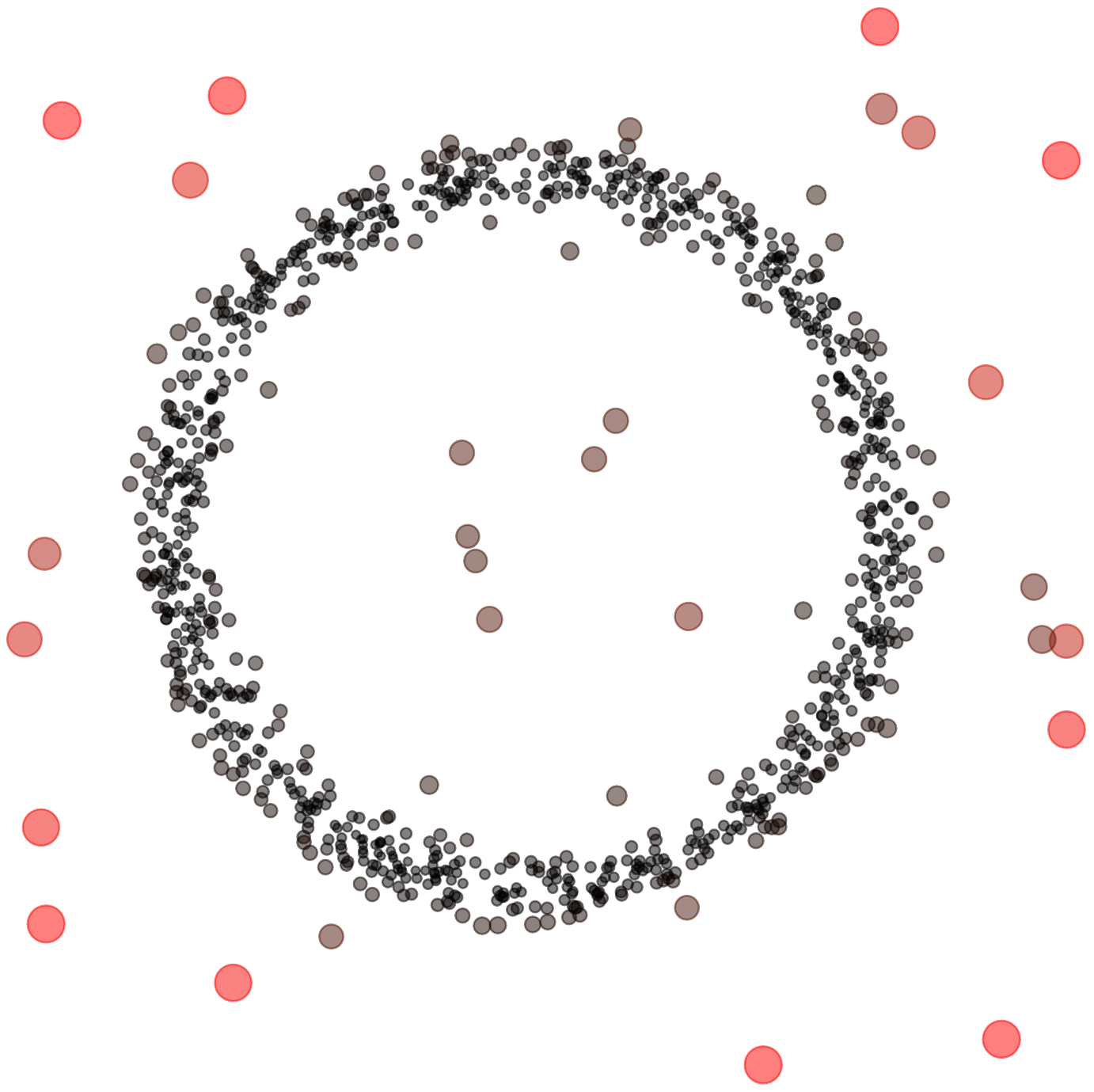}
	\caption{Left: 1000 points in $\RR^2$ and the level sets of the corresponding inverse moment matrix SOS polynomial $Q_{\mu,d}$ ($d = 4$). The level set ${p+d \choose d}$, which corresponds to the average value of $Q_{\mu,d}$, is represented in red. Right: 1040 points in $\RR^2$ with size and color proportional to the value of inverse moment matrix SOS polynomial $Q_{\mu,d}$ ($d=8$).}
	\label{fig:examples}
\end{figure}
\section{Multivariate polynomials and moments}
\label{sec:notations}
\subsection{Notations}
We fix the ambient dimension to be $p$ throughout the text. For example, we will manipulate vectors in $\RR^p$ as well as $p$-variate polynomials with real coefficients. We denote by $X$ a set of $p$ variables $X_1, \ldots, X_p$ which we will use in mathematical expressions defining polynomials. We identify monomials from the canonical basis of $p$-variate polynomials with their exponents in $\NN^p$: we associate to $\B{\alpha} = (\alpha_i)_{i =1 \ldots p} \in \NN^p$ the monomial $X^{\B{\alpha}} := X_1^{\B{\alpha}_1} X_2^{\B{\alpha}_2} \ldots X_p^{\B{\alpha}_p}$ which degree is $\deg(\B{\alpha}) := \sum_{i=1}^p \alpha_i$. We use the expressions $\lgl$ and $\leqgl$ to denote the graded lexicographic order, a well ordering over $p$-variate monomials. This amounts to, first, use the canonical order on the degree and, second, break ties in monomials with the same degree using the lexicographic order with $X_1 =a, X_2 = b \ldots$ For example, the monomials in two variables $X_1, X_2$, of degree less or equal to $3$ listed in this order are given by: $1,\,X_1, \,X_2, \,X_1^2, \,X_1X_2, \,X_2^2,\, X_1^3,\, X_1^2X_2,\, X_1X_2^2,\, X_2^3$.

We denote by $\NN^p_d$, the set $\left\{ \alpha \in \NN^p;\; \deg(\alpha) \leq d \right\}$ ordered by $\leqgl$. $\RR[X]$ denotes the set of $p$-variate polynomials: linear combinations of monomials with real coefficients. The degree of a polynomial is the highest of the degrees of its monomials with nonzero coefficients\footnote{For the null polynomial, we use the convention that its degree is $0$ and it is $\leqgl$ smaller than all other monomials.}. We use the same notation, $\deg(\cdot)$, to denote the degree of a polynomial or of an element of $\NN^p$. For $d \in \NN$, $\RR_d[X]$ denotes the set of $p$-variate polynomials of degree less or equal to $d$. We set $s(d) = {p+d \choose d}$, the number of monomials of degree less or equal to $d$. We will denote by $\B{v}_d(X)$ the vector of monomials of degree less or equal to $d$ sorted by $\leqgl$. We let $\B{v}_d(X) := \left( X^\alpha \right)_{\alpha \in \NN^p_d} \in \RR_d[X]^{s(d)}$. With this notation, we can write a polynomial $P\in\RR_d[X]$ as follows $P(X) = \left\langle\B{p}, \B{v}_d(X)\right\rangle$ for some real vector of coefficients $\B{p} = \left( p_{\B{\alpha}} \right)_{\B{\alpha} \in \NN_d^p} \in \RR^{s(d)}$ ordered using $\leqgl$. Given $\B{x} = (x_i)_{i = 1 \ldots p} \in \RR^p$, $P(\B{x})$ denotes the evaluation of $P$ with the assignments $X_1 = x_1, X_2 = x_2, \ldots X_p = x_{p}$. Given a Borel probability measure $\mu$ and $\B{\alpha} \in \NN^p$, $y_{\B{\alpha}}(\mu)$ denotes the moment $\B{\alpha}$ of $\mu$: $y_{\B{\alpha}}(\mu) = \int_{\RR^p} \B{x}^{\B{\alpha}} d\mu(\B{x})$. Throughout the paper, we will only consider measures of which all moments are finite.
\subsection{Moment matrix}
Given a Borel probability measure $\mu$ on $\RR^p$, the moment matrix of $\mu$, $M_d(\mu)$, is a matrix indexed by monomials of degree at most $d$ ordered by $\leqgl$. For $\B{\alpha},\B{\beta} \in \NN^p_d$, the corresponding entry in $M_d(\mu)$ is defined by $M_d(\mu)_{\B{\alpha},\B{\beta}} := y_{\B{\alpha} + \B{\beta}}(\mu)$, the moment $\B{\alpha} + \B{\beta}$ of $\mu$. When $p = 2$, letting $y_{\B{\alpha}} = y_{\B{\alpha}} (\mu)$ for $\B{\alpha} \in \NN_4^2$, we have
\[M_2(\mu): \quad\begin{array}{rccccccccc}
& & &1& X_1  & X_2  & X_1^2 & X_1 X_2 &  X_2^2 \\
 & & & \\
 1  &   \quad & &1 & y_{10} & y_{01} & y_{20}& y_{11}& y_{02} \\
X_1&   \quad &&y_{10} & y_{20} & y_{11}&y_{30} & y_{21}& y_{12} \\
X_2&   \quad & &y_{01} & y_{11} & y_{02}& y_{21} &y_{12}& y_{03} \\
X_1^2& \quad & &y_{20} & y_{30} & y_{21}& y_{40} & y_{31}& y_{22} \\
X_1X_2& \quad & &y_{11} & y_{21} & y_{12}& y_{31} & y_{22}& y_{13}\\
X_2^2& \quad & &y_{02} & y_{12} &y_{03}& y_{22} & y_{13}& y_{04}\\
\end{array}.\]
$M_d(\mu)$ is positive semidefinite for all $d \in \NN$. Indeed, for any $\B{p} \in \RR^{s(d)}$, let $P \in \RR_d[X]$ be the polynomial with vector of coefficients $\B{p}$, we have $\B{p}^TM_d(\mu)\B{p} = \int_{\RR^p} P^2(\B{x}) d\mu(\B{x}) \geq 0$. Furthermore, we have the identity $M_d(\mu) = \int_{\RR^p} \B{v}_d(\B{x}) \B{v}_d(\B{x})^T d\mu(\B{x})$ where the integral is understood elementwise.
\subsection{Sum of squares (SOS)}
We denote by $\Sigma[X] \subset \RR[X]$ (resp. $\Sigma_d[X] \subset \RR_d[X]$), the set of polynomials (resp. polynomials of degree at most $d$) which can be written as a sum of squares of polynomials. Let $P \in \RR_{2m}[X]$ for some $m \in \NN$, then $P$ belongs to $\Sigma_{2m}[X]$ if there exists a finite $J \subset \NN$ and a family of polynomials $P_j \in \RR_{m}[X]$, $j \in J$, such that $P = \sum_{j \in J} P_j^2$. It is obvious that sum of squares polynomials are always nonnegative. A further interesting property is that this class of polynomials is connected with positive semidefiniteness. Indeed, $P$ belongs to $\Sigma_{2m}[X]$ if and only if
\begin{align}
	\exists Q \in \RR^{s(m) \times s(m)},\, Q\succeq 0,\, P(\B{x}) = \B{v}_d(\B{x})^T Q \B{v}_d(\B{x}),\, \forall \B{x} \in \RR^p.
	\label{eq:equivSOS}
\end{align}
As a consequence, every positive semidefinite matrix $Q \in \RR^{s(m) \times s(m)}$ defines a polynomial in $\Sigma_{2m}[X]$ by using the representation in (\ref{eq:equivSOS}).

\section{Empirical observations on the inverse moment matrix SOS polynomial}
\label{sec:empirical}
 The {\it inverse moment-matrix SOS polynomial} is associated to a measure $\mu$ which satisfies the following.
\begin{assumption}
	\label{ass:mainAssumption}
	$\mu$ is a Borel probability measure on $\RR^p$ with all its moments finite and $M_d(\mu)$ is positive definite for a given $d \in \NN$.
\end{assumption}
\begin{definition}
	\label{def:SOSpoly}
	Let $\mu,d$ satisfy Assumption \ref{ass:mainAssumption}. We call the SOS polynomial $Q_{\mu,d} \in \Sigma_{2d}[X]$ defined by the application:
	\begin{align}
	\label{pol-def}
		\B{x}\mapsto\quad Q_{\mu,d}(\B{x})\, :=\, \B{v}_d(\B{x})^T M_d(\mu)^{-1} \B{v}_d(\B{x}),\qquad  \B{x} \in \RR^p,
	\end{align}
the inverse moment-matrix SOS polynomial of degree $2d$ associated to $\mu$.
\end{definition}
Actually, connection to orthogonal polynomials will show that the inverse function $\B{x}\mapsto Q_{\mu,d}(\B{x})^{-1}$ is called the {\it Christoffel} function in the literature \cite{szego1974orthogonal,dunkl2001orthogonal} (see also Section \ref{sec:interpretation}). 

In the remainder of this section, we focus on the situation when $\mu$ corresponds to an empirical measure over $n$ points in $\RR^p$ which are fixed. So let $\B{x}_1, \ldots, \B{x}_n \in \RR^p$ be a fixed set of points and let $\mu := \frac{1}{n} \sum_{i=1}^n \delta_{\B{x}_i}$ where $\delta_{\B{x}}$ corresponds to the Dirac measure 
at $\B{x}$. In such a case the polynomial $Q_{\mu,d}$ in (\ref{pol-def}) is determined only by the empirical moments up to degree $2d$ of our collection of points. Note that we also require that $M_d(\mu) \succ 0$. In other words, the points $\B{x}_1, \ldots, \B{x}_n$ do not belong to an algebraic set defined by a polynomial of degree less or equal to $d$. We first describe empirical properties of inverse moment matrix SOS polynomial in this context of empirical measures. A mathematical intuition and further properties behind these observations are developped in Section \ref{sec:interpretation}.
\subsection{Sublevel sets}
The starting point of our investigations is the following phenomenon which to the best of our knowledge
has remained unnoticed in the literature. For the sake of clarity and simplicity we provide an illustration in the plane. Consider the following experiment in $\RR^2$ for a fixed $d \in \NN$: represent on the same graphic, the cloud of points $\left\{ \B{x}_i \right\}_{i = 1 \ldots n}$ and the sublevel sets of SOS polynomial $Q_{\mu,d}$ in $\RR^2$ (equivalently, the superlevel sets of the Christoffel function). This is illustrated in the left panel of Figure \ref{fig:examples}. 
The collection of points consists of $500$ simulations of two different Gaussians and the value of $d$ is $4$. The striking feature of this plot is that the level sets capture the global shape of the cloud of points quite accurately. In particular, the level set $\{\B{x}:Q_{\mu,d}(\B{x})\leq {p+d\choose d}\}$ captures most of the points. We could reproduce very similar observations on different shapes with various number of points in $\RR^2$ and degree $d$ (see Appendix A).


\subsection{Measuring outlyingness}
An additional remark in a similar line is that $Q_{\mu, d}$ tends to take higher values on points which are isolated from other points. Indeed in the left panel of Figure \ref{fig:examples}, the value of the polynomial tends to be smaller on the boundary of the cloud. This extends to situations where the collection of points correspond to shape with a high density of points with a few additional outliers. We reproduce a similar experiment on the right panel of Figure \ref{fig:examples}. In this example, $1000$ points are sampled close to a ring shape and 40 additional points are sampled uniformly on a larger square. We do not represent the sublevel sets of $Q_{\mu,d}$ here. Instead, the color and shape of the points are taken proportionally to the value of $Q_{\mu,d}$, with $d = 8$.

First, the results confirm the observation of the previous paragraph, points that fall close to the ring shape tend to be smaller and points on the boundary of the ring shape are larger. Second, there is a clear increase in the size of the points that are relatively far away from the ring shape. This highlight the fact that $Q_{\mu,d}$ tends to take higher value in less populated areas of the space. 
\subsection{Relation to maximum likelihood estimation}
\label{sec:maxLikelihood}
If we fix $d=1$, we recover the maximum likelihood estimation for the Gaussian, up to a constant additive factor. To see this, set $\mu = \frac{1}{n} \sum_{i=1}^n \B{x}_i$ and $S = \frac{1}{n} \sum_{i=1}^n \B{x}_i \B{x}_i^T$. With this notation, we have the following block representation of the moment matrix,
\[M_d(\mu) = \left( 
\begin{array}{cc}
	1&\mu^T\\
	\mu &S
\end{array}
\right)\quad\quad M_{d}(\mu)^{-1} = \left( 
\begin{array}{cc}
	1 + \mu^T V^{-1}\mu & - \mu^T V^{-1}\\
	- V^{-1} \mu & V^{-1}
\end{array}
\right),\]
where $V = S - \mu \mu^T$ is the empirical covariance matrix and the expression for the inverse is given by Schur complement. In this case, we have $Q_{\mu,1}(\B{x}) = 1 + (\B{x} - \mu)^T V^{-1}(\B{x} - \mu)$ for all $\B{x} \in \RR^p$. 
We recognize the quadratic form that appears in the density function of the multivariate Gaussian with parameters estimated by maximum likelihood. This suggests a connection between the inverse SOS moment polynomial and maximum likelihood estimation. Unfortunately, this connection is difficult to generalize for higher values of $d$ and we do not pursue the idea of interpreting the empirical observations of this section through the prism of maximum likelihood estimation and leave it for further research. Instead, we propose an alternative view in Section \ref{sec:interpretation}.
\subsection{Computational aspects}
Recall that $s(d) = {p+d \choose d}$ is the number of $p$-variate monomials of degree up to $d$. The computation of $Q_{\mu,d}$ requires $O(n s(d)^2)$ operations for the computation of the moment matrix and $O(s(d)^3)$ operations for the matrix inversion. The evaluation of $Q_{\mu,d}$ requires $O(s(d)^2)$ operations.

Estimating the coefficients of $Q_{\mu,d}$ has a computational cost that depends only linearly in the number of points $n$. The cost of evaluating $Q_{\mu,d}$ is constant with respect to the number of points $n$. This is an important contrast with kernel based or distance based methods (such as nearest neighbors and one class SVM) for density estimation or outlier detection since they usually require at least $O(n^2)$ operations for the evaluation of the model \cite{chandola2009anomaly}. Moreover, this is well suited for online settings where inverse moment matrix computation can be done using Woodbury updates.

The dependence in the dimension $p$ is of the order of $p^d$ for a fixed $d$. Similarly, the dependence in $d$ is of the order of $d^p$ for a fixed dimension $p$ and the joint dependence is exponential. This suggests that the computation and evaluation of $Q_{\mu,d}$ will mostly make sense for moderate dimensions and degree $d$.
\section{Invariances and interpretation through orthogonal polynomials}
\label{sec:interpretation}
The purpose of this section is to provide a mathematical rationale that  {\it explains} the empirical observations made in Section \ref{sec:empirical}. All the proofs are postponed to Appendix B. We fix a Borel probability measure $\mu$ on $\RR^p$ which satisfies Assumption \ref{ass:mainAssumption}.
Note that $M_d(\mu)$ is always positive definite if $\mu$ is not supported on the zero set of a polynomial of degree at most $d$. Under Assumption \ref{ass:mainAssumption}, $M_d(\mu)$ induces an inner product on $\RR^{s(d)}$ and by extension on $\RR_d[X]$ (see Section \ref{sec:notations}). This inner product is denoted by $\left\langle\cdot,\cdot\right\rangle_{\mu}$ and satisfies for any polynomials $P,Q \in \RR_d[X]$ with coefficients $\B{p}, \B{q} \in \RR^{s(d)}$,
\begin{align*}
	\left\langle P, Q\right\rangle_\mu := \left\langle\B{p}, M_d(\mu)\B{q}\right\rangle_{\RR^{s(d)}} = \int_{\RR^p} P(\B{x})Q(\B{x}) d\mu(\B{x}).
\end{align*}
We will also use the canonical inner product over $\RR_d[X]$ which we write $\left\langle P, Q\right\rangle_{\RR_d[X]} := \left\langle\B{p}, \B{q}\right\rangle_{\RR^{s(d)}}$ for any polynomials $P,Q \in \RR_d[X]$ with coefficients $\B{p}, \B{q} \in \RR^{s(d)}$. We will omit the subscripts for this canonical inner product and use $\left\langle\cdot,\cdot\right\rangle$ for both products.

\subsection{Affine invariance}
It is worth noticing that the mapping $\B{x} \mapsto Q_{\mu,d}(\B{x})$ does not depend on the particular choice of $\B{v}_d(X)$ as a basis of $\RR_d[X]$, any other basis would lead to the same mapping. This leads to the result that $Q_{\mu,d}$ captures affine invariant properties of $\mu$.
\begin{lemma}
	\label{lem:affineInvariance}
	Let $\mu$ satisfy Assumption \ref{ass:mainAssumption} and $A \in \RR^{p\times p}, b\in \RR^p$ define an invertible affine mapping on $\RR^p$, $\mathcal{A}\colon\B{x} \to A \B{x} + b$. Then, the push foward measure, defined by $\tilde{\mu}(S) = \mu(\mathcal{A}^{-1}(S))$ for all Borel sets $S \subset \RR^p$, satisfies Assumption \ref{ass:mainAssumption} (with the same $d$ as $\mu$) and for all $\B{x} \in \RR^p$, $Q_{\mu,d}(\B{x}) = Q_{\tilde{\mu},d}(A\B{x} + b)$.
\end{lemma}
Lemma \ref{lem:affineInvariance} is probably better understood when $\mu = 1/n \sum_{i=1}^n \delta_{\B{x}_i}$ as in Section \ref{sec:empirical}. In this case, we have $\tilde{\mu} = 1/n \sum_{i=1}^n \delta_{A \B{x}_i + b}$ and Lemma \ref{lem:affineInvariance} asserts that the level sets of $Q_{\tilde{\mu},d}$ are simply the images of those of $Q_{\mu,d}$ under the affine transformation $\B{x} \mapsto A \B{x} + b$. This is illustrated in Appendix D. 
\subsection{Connection with orthogonal polynomials}
We define a classical \cite{szego1974orthogonal,dunkl2001orthogonal} family of orthonormal polynomials, $\left\{ P_\alpha \right\}_{\alpha \in \NN^p_d}$ ordered according to $\leqgl$ which satisfies for all $\alpha \in \NN^p_d$
\begin{equation}
\langle P_\alpha, X^\beta \rangle = 0\ {\rm if} \ \alpha \lgl \beta,\, \langle P_\alpha,P_\alpha\rangle_\mu=1,\,\langle P_\alpha, X^\beta \rangle_\mu = 0\ {\rm if} \ \beta \lgl \alpha,
\,\langle P_\alpha, X^\alpha \rangle_\mu > 0.
\label{eq:orthoDef}
\end{equation}
It follows from (\ref{eq:orthoDef}) that $\left\langle P_\alpha, P_\beta \right\rangle_\mu = 0$ if $\alpha \neq \beta$. Existence and uniqueness of such a family is guaranteed by the Gram-Schmidt orthonormalization process following the $\leqgl$ order on the monomials, and by the positivity of the moment matrix, see for instance \cite[Theorem 3.1.11]{dunkl2001orthogonal}.

Let $D_d(\mu)$ be the lower triangular matrix which rows are the coefficients of the polynomials $P_\alpha$ defined in (\ref{eq:orthoDef}) ordered by $\leqgl$. It can be shown that $D_d(\mu) = L_d(\mu)^{-T}$, where $L_d(\mu)$ is the Cholesky factorization of $M_d(\mu)$. Furthermore, there is a direct relation with the inverse moment matrix as $M_d(\mu)^{-1} = D_d(\mu)^T D_d(\mu)$ \cite[Proof of Theorem 3.1]{helton2008measure}. This has the following consequence.
\begin{lemma}
	\label{lem:connectionOrtho}
	Let $\mu$ satisfy Assumption \ref{ass:mainAssumption}, then $Q_{\mu,d} = \sum_{\alpha \in \NN^p_d} P_\alpha^2$, where the family $\left\{P_\alpha  \right\}_{\alpha \in \NN^p_d}$ is defined by (\ref{eq:orthoDef}) and $\int_{\RR^p} Q_{\mu,d}(\B{x})d\mu(\B{x}) = s(d)$.
\end{lemma}
That is, $Q_{\mu,d}$ is a very specific and distinguished SOS polynomial, the sum of squares of the orthonormal basis elements $\left\{P_\alpha  \right\}_{\alpha \in \NN_d^p}$ 
of $\mathbb{R}_d(X)$ (w.r.t. $\mu$). Furthermore, the average value of $Q_{\mu,d}$ with respect to $\mu$ is $s(d)$ which corresponds to the red level set in left panel of Figure \ref{fig:examples}.

\subsection{A variational formulation for the inverse moment matrix SOS polynomial}
In this section, we show that the family of polynomials $\left\{P_\alpha \right\}_{\alpha \in \NN^p_d}$ defined in (\ref{eq:orthoDef}) is the unique solution (up to a multiplicative constant) of a convex optimization problem over polynomials. This fact combined with Lemma \ref{lem:connectionOrtho} provides a mathematical rationale for the empirical observations outlined in Section \ref{sec:empirical}. Consider the following optimization problem.
\begin{align}
	\label{eq:optPoly}
	\min_{Q_\alpha, \theta_\alpha, \alpha \in \NN^p_d}\;&\;\frac{1}{2}\int_{\RR^p} \sum_{\alpha \in \NN^p_d} Q_{\alpha}(\B{x})^2 d\mu(\B{x})\\
	\mathrm{s.t.}\;&\; \B{q}_{\alpha\alpha}\, \geq \exp(\theta_{\alpha}),\quad\B{q}_{\alpha\beta} = 0,\quad \alpha, \beta \in \NN^p_d,\quad \alpha \lgl \beta,\quad\sum_{\alpha \in \NN^p_d} \theta_\alpha= 0,\nonumber
\end{align}
where $Q_\alpha(\B{x})=\sum_{\beta \in \NN_d^p} \B{q}_{\alpha\beta}\B{x}^\beta$, $\alpha\in\NN^p_d$.
We first comment on problem (\ref{eq:optPoly}). Let $P = \sum_{\alpha \in \NN^p_d} Q_{\alpha}^2$ be the SOS polynomial appearing in the objective function of (\ref{eq:optPoly}). The constraints of problem (\ref{eq:optPoly}) restrict $P$ to be in a certain set $S_d \subset \Sigma_d[X]$. With this notation, problem (\ref{eq:optPoly}) is reformulated as $\min_{P\in S_d}\int P d\mu$. Therefore problem (\ref{eq:optPoly}) balances two antagonist targets, on one hand the minimization of the average value of the SOS polynomial $P$ with respect to $\mu$, on the other hand the avoidance of the trivial polynomial, enforced by the constraint that $P \in S_d$. The constraints on $P$ are simple and natural, they ensure that $P$ is a sum of squares of polynomials $\left\{Q_\alpha  \right\}_{\alpha \in \NN_d^p}$, where the leading term of $Q_\alpha$ (according to the ordering $\leqgl$) is $\B{q}_{\alpha\alpha}\B{x}^\alpha$ with $\B{q}_{\alpha\alpha}>0$ (and hence does not vanish). Inversely, using Cholesky factorization, for any SOS polynomial $Q$ of degree $2d$ which coefficient matrix (see equation (\ref{eq:equivSOS})) is positive definite, there exists $a >0$ such that	$aQ \in S_d$. This suggests that $S_d$ is a quite general class of nonvanishing SOS polynomials. The following result, which gives a relation between $Q_{\mu,d}$ and solutions of (\ref{eq:optPoly}), uses a generalization of \cite[Theorem 3.1.2]{szego1974orthogonal} to several orthogonal polynomials of several variables.
\begin{theorem}
  \label{prop:mainRes}:
	Under Assumption \ref{ass:mainAssumption}, problem (\ref{eq:optPoly}) is a convex optimization problem with a unique optimal solution $(Q^*_\alpha,\theta_\alpha^*)$, which satisfies $Q^*_\alpha= \sqrt{\lambda}P_\alpha$, $\alpha\in\NN^p_d$, for some $\lambda>0$. In particular,
		the distinguished SOS  polynomial 
		$Q_{\mu,d} \,=\,\sum_{\alpha\in\NN^p_d}P_\alpha^2\,=\, \frac{1}{\lambda} \sum_{\alpha \in \NN^p_d} (Q^*_{\alpha})^2,$
		is (part of) the unique optimal solution of (\ref{eq:optPoly}).
\end{theorem}

Theorem \ref{prop:mainRes} states that up to the scaling factor $\lambda$, the distinguished SOS polynomial $Q_{\mu,d}$ is {\it the unique optimal solution of problem (\ref{eq:optPoly}).} A detailed proof is provided in the Appendix B and we only sketch the main ideas here. First, it is remarkable that for each fixed $\alpha\in\NN^p_d$ (and again up to a scaling factor) the polynomial $P_\alpha$ is the unique optimal solution of the problem: $\min_Q\:\left\{\:\int Q^2 d\mu\,:\:Q \in \RR_d[X],\,Q(\B{x})=\B{x}^\alpha+\sum_{\beta \lgl \alpha}\B{q}_\beta\,\B{x}^\beta\:  \right\}$.
		This fact is well-known in the univariate case \cite[Theorem 3.1.2]{szego1974orthogonal} and does not seem to have been exploited in the literature, at least for purposes similar to ours. So intuitively, $P_\alpha^2$ should be as close to $0$ as possible on the support of $\mu$. Problem (\ref{eq:optPoly}) has similar properties and
the constraint on the vector of weights $\theta$ enforces that, at an optimal solution, the contribution $\int (Q^*_\alpha)^2\,d\mu$ to the overall sum in the criterion is the same for all $\alpha$. Using Lemma \ref{lem:connectionOrtho} yields (up to a multiplicative constant) the polynomial $Q_{\mu,d}$. Other constraints
on $\theta$ would yield different weighted sum of the
squares $P_\alpha^2$. This will be a subject of further investigations.

To sum up, Theorem \ref{prop:mainRes} provides a rationale for our observations. Indeed when solving (\ref{eq:optPoly}), intuitively, $Q_{\mu,d}$ should be as close to $0$ as possible on average while remaining in a large class of nonvanishing SOS polynomials.

		\subsection{Christoffel function and outlier detection}
		The following result from \cite[Theorem 3.5.6]{dunkl2001orthogonal} draws a direct connection between $Q_{\mu,d}$ and the \textit{Chritoffel} function (the right hand side of (\ref{eq:christoffel})).
		\begin{theorem}[\cite{dunkl2001orthogonal}]
		\label{th:christoffel}
		Let Assumption \ref{ass:mainAssumption} hold and let $\bar{\B{x}}\in\RR^p$ be fixed, arbitrary. Then
		\begin{equation}
			\label{eq:christoffel}
			Q_{\mu,d}(\bar{\B{x}})^{-1}\,=\,\min_{P\in\RR_d[X]}\:\left\{\int_{\RR^p} P(\B{x})^2\,d\mu(\B{x}): P(\bar{\B{x}})\,=\,1 \right\}.
		\end{equation}
		\end{theorem}
		Theorem \ref{th:christoffel} provides a mathematical rationale for the use of $Q_{\mu,d}$ for outlier or novelty detection purposes. Indeed, from Lemma \ref{lem:connectionOrtho} and equation (\ref{eq:orthoDef}), we have $Q_{\mu,d} \geq 1$ on $\RR^p$. Furthermore, the solution of the minimization problem in (\ref{eq:christoffel}) satisfies $P(\bar{\B{x}})^2 = 1$ and $\mu\left( \left\{ \B{x} \in \RR^p:\: P(\B{x})^2 \leq 1 \right\} \right) \geq 1-Q_{\mu,d}(\bar{\B{x}})^{-1}$ (by Markov's inequality). Hence, for high values of $Q_{\mu,d}(\bar{\B{x}})$, the sublevel set 
		$\left\{ \B{x} \in \RR^p:\: P(\B{x})^2 \leq 1 \right\}$ contains most of the mass of $\mu$ while $P(\bar{\B{x}})^2 = 1$. 
Again the result of Theorem \ref{th:christoffel} does not seem to have been interpreted for purposes similar to ours.
\section{Experiments on network intrusion datasets}
\label{sec:experiments}

In addition to having its own mathematical interest, Theorem \ref{prop:mainRes} can be exploited for various purposes. For instance, the sub-level sets of $Q_{\mu,d}$, and in particular $\{\B{x}\in\mathbb{R}^{p}: Q_{\mu,d}(\B{x})\leq {p+d\choose d}\}$, can be used  to encode a cloud of points in a simple and compact form. However in this section we focus on another potential application in anomaly detection.

Empirical findings described in Section \ref{sec:empirical} suggest that the polynomial $Q_{\mu,d}$ can be used to detect outliers in a collection of real vectors by taking $\mu$ to be the corresponding empirical measure. This is backed up by the results presented in Section \ref{sec:interpretation}. In this section we illustrate these properties on a real world example. We choose the KDD cup 99 network intrusion dataset (available at \cite{lichmanUCI2013}) which consists of network connection data with labels describing whether they correspond to normal traffic or network intrusions. We follow \cite{yamanishi2004online} and \cite{williams2002comparative} and construct five datasets consisting of labeled vectors in $\RR^3$, the label indicating normal traffic or network attack. The content of these datasets is summarized in the following table. 
\begin{center}
	\begin{tabular}{c|c|c|c|c|c}
		Dataset&http&smtp&ftp-data&ftp&others\\\hline
		Number of examples&567498&95156&30464&4091&5858\\
		Proportions of attacks&0.004&0.0003&0.023&0.077&0.016
	\end{tabular}
\end{center}

The details on how these datasets are constructed are available in \cite{yamanishi2004online,williams2002comparative} and are reproduced in Appendix C. The main idea is to give to each datapoint an outlyingness score solely based on its position in $\RR^3$ and then compare outliers predicted by the score with the label indicating network intrusion. The underlying assumption is that network intrusion corresponds to infrequent abnormal behaviors and could thus be considered as outliers.

We reproduce the exact same experiment that was described in \cite[Section 5.4]{williams2002comparative} using the value of the inverse moment matrix SOS polynomial from Definition \ref{def:SOSpoly} as an outlyingness score (with $d=3$). The authors of \cite{williams2002comparative} have compared different types of methods for outlier detection in the same experimental setting: methods based on robust estimation and Mahalanobis distance \cite{hadi1994modification,knorr2001robust}, mixture model based methods \cite{olivier1996} and recurrent neural network based methods. These results are gathered in \cite[Figure 7]{williams2002comparative}. In the left panel of Figure \ref{fig:KDD} we represent the same performance measure for our approach. We first compute the value of the inverse moment SOS polynomial for each datapoint and use it as an outlyingness score. We then display the proportion of correctly identified outliers, with score above a given threshold, as a function of the proportion of examples with score above the threshold (for different values of the threshold). The main comments are as follows.

$\bullet$ The inverse moment matrix SOS polynomial does detect network intrusions with varying performances on the five datasets.

$\bullet$ Except for the ``ftp-data dataset'', the global shape of these curves are very similar to results reported in \cite[Figure 7]{williams2002comparative} indicating that the proposed approach is comparable to other dedicated methods for intrusion detection in these four datasets.

In a second experiment, we investigate the effect of changing the value of $d$ in $Q_{\mu,d}$ on the performances in terms of outlier detection. We focus on the ``others'' dataset because it is the most heterogeneous in term of data and outliers. We adopt a slightly different measure of performance and use precision recall curves (see for example \cite{davis2006relationship}) to measure performances in identifying network intrusions (the higher the curve, the better). We call the area under such curves the AUPR. The right panel of Figure \ref{fig:KDD} represents these results. First, the case $d = 1$, which corresponds to vanilla Mahalanobis distance as outlined in Section \ref{sec:maxLikelihood}, gives poor performances. Second, the global performances rapidly increase with $d$ and then decrease and stabilize.

This suggests that $d$ can be used as a tuning parameter which controls the ``complexity'' of $Q_{\mu,d}$. Indeed, $2d$ is the degree of the polynomial $Q_{\mu,d}$  and it is expected that more complex models will potentially identify more diverse classes of examples of points as outliers. In our case, this means identifying regular traffic as outliers while it actually does not correspond to intrusions.

\begin{figure}[t]
	\centering
	\includegraphics[width=.47\textwidth]{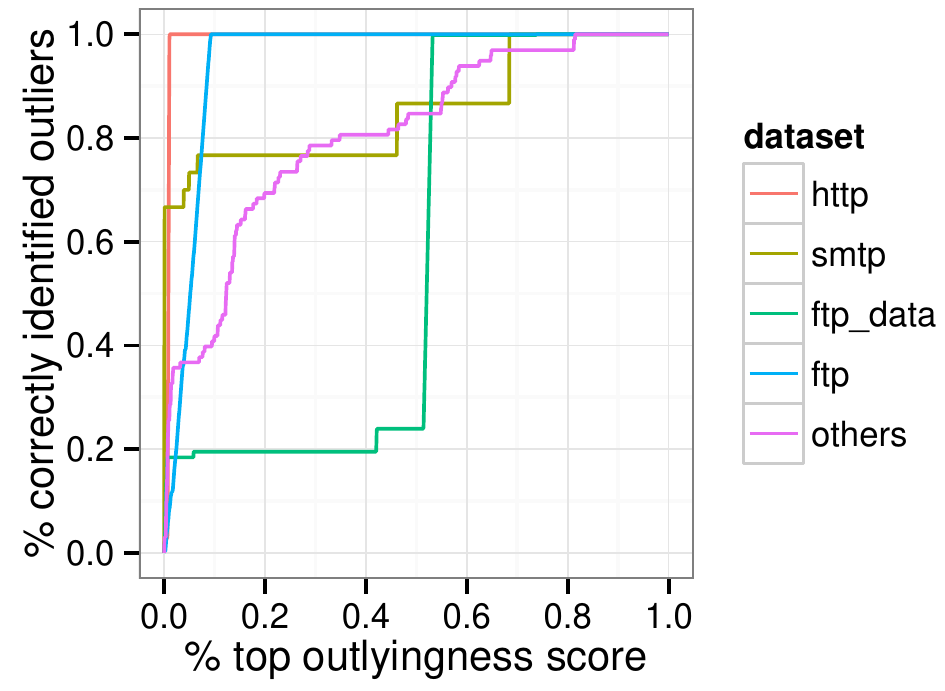}\quad
	\includegraphics[width=.47\textwidth]{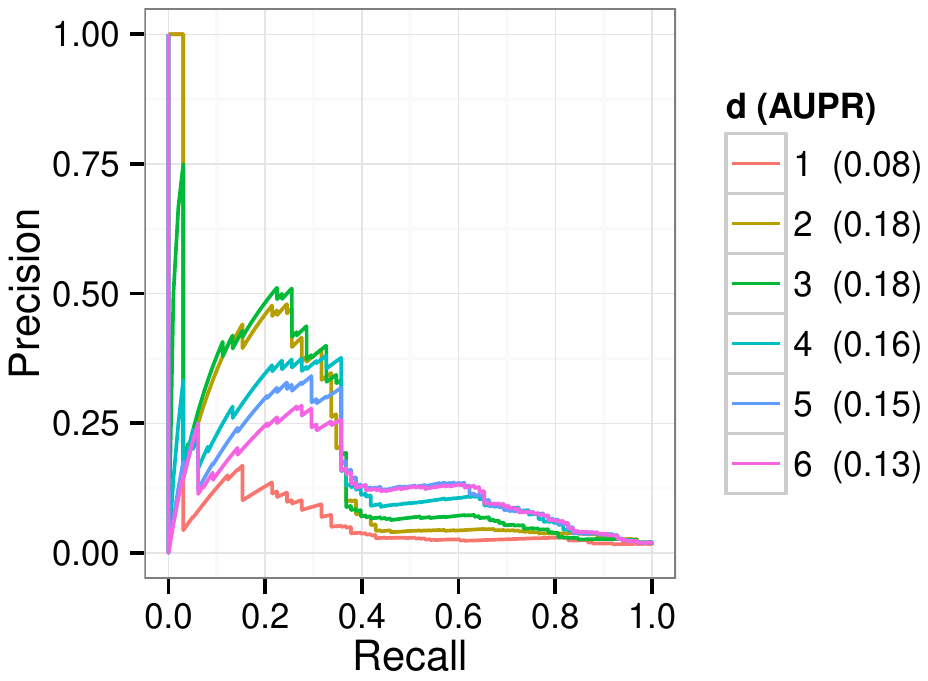}
	\caption{Left: reproduction of the results described in \cite{williams2002comparative} with the inverse moment SOS polynomial value as an outlyingness score ($d=3$). Right: Precision-recall curves for different values of $d$ for the dataset ``others''.}
	\label{fig:KDD}
\end{figure}
\section{Conclusion and future work}
We presented empirical findings with a mathematical intuition regarding the sublevel sets of the inverse moment matrix SOS polynomial. This opens many potential subjects of investigations.
\begin{itemize}
	\item Similarities with maximum likelihood. 
	\item Statistics in the context of empirical processes.
	\item Relation between a density and its inverse moment matrix SOS polynomial. Assymptotics when the degree increases.
	\item Connections with computational geometry and non Gaussian integrals.
	\item Computationally tractable extensions in higher dimensions.
\end{itemize}

\subsubsection*{Acknowledgments}
This work was partly supported by project ERC-ADG TAMING 666981,ERC-Advanced Grant of the {\it European Research Council} and grant number FA9550-15-1-0500 from the {\it Air Force Office of Scientific Research, Air Force Material Command}.

%

\bibliographystyle{plain}

\newpage
\appendix
\section{Additional examples}
\begin{figure}[h!]
	\centering
	\includegraphics[width=\textwidth]{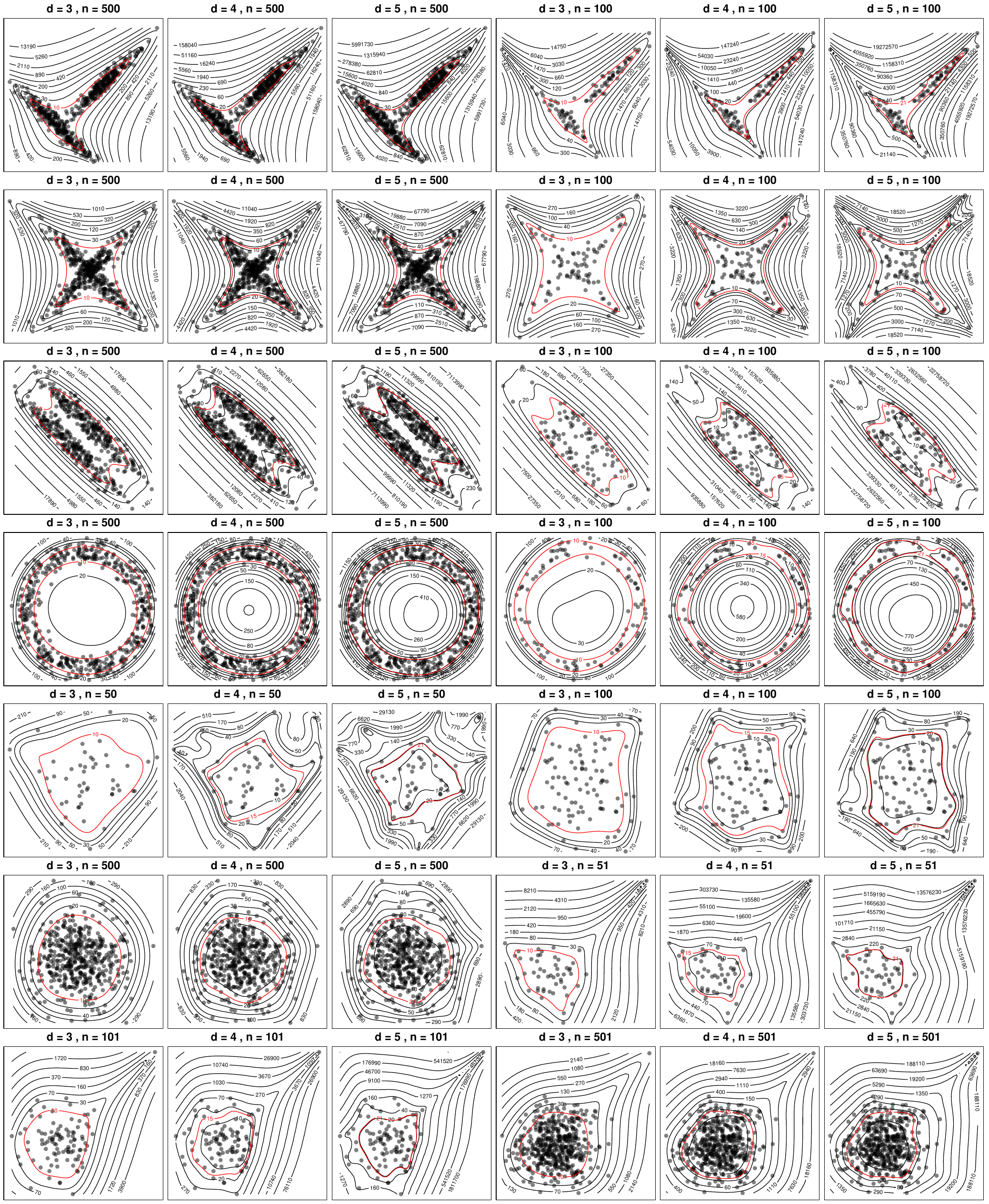}\quad\quad
	\caption{Empirical measure and level sets of $Q_{\mu,d}$ for various values of $d$ and various configurations and numbers of points $n$. The level set ${p+d \choose d}$, which corresponds to the average value of $Q_{\mu,d}$, is represented in red.}
	\label{fig:examples}
\end{figure}
\newpage
\section{Proofs}
\setcounter{equation}{6}
We use the same notation as in the main text. We recall that Assumption 1. 
\setcounter{assumption}{0}
\begin{assumption}
	\label{ass:mainAssumptionAppendix}
	$\mu$ is a Borel probability measure on $\RR^p$ with all its moments finite and $M_d(\mu)$ is positive definite for a given $d \in \NN$.
\end{assumption}
Lemma 2 and Theorem 2 are taken from the literature and we provide a proof for completeness.
\subsection{Proof of Lemma 1}
First we show that the mapping $\B{x} \mapsto Q_{\mu,d}(\B{x})$ does not depend on the choice of a specific basis of $\RR_d[X]$. Then we will deduce the affine invariance property.
\begin{lemma}
	\label{lem:auxLemma1Appendix}
	Let $\B{w}_d(X)$ be an arbitrary basis of $\RR_d[X]$ and let $R_{\mu,d} \in \RR_d[X]$ be derived in the same way as $Q_{\mu,d}$ (see Definition 1), with $\B{w}_d$ in place of $\B{v}_d$. Then $Q_{\mu,d}(\B{x}) = R_{\mu,d}(\B{x})$ for all $\B{x} \in \RR^p$.
\end{lemma}
\begin{proof}
	Since $\B{w}_d$ is a basis of $\RR_d[X]$, there exists an invertible matrix $C \in \RR^{s(d) \times s(d)}$ such that $\B{w}_d(X) = C \B{v}_d(X)$. We reproduce the computation of Definition 1 with this new basis. We write $N_{d}(\mu)$ the moment matrix computed with the polynomial basis $\B{w}_d$. We have
	\begin{align*}
		N_d(\mu) &= \int_{\RR^p} \B{w}_d(\B{x})\B{w}_d(\B{x})^T d\mu(\B{x})\\
		&=\int_{\RR^p} C \B{v}_d(\B{x})\B{v}_d(\B{x})^TC^T d\mu(\B{x})\\
		&=C\int_{\RR^p} \B{v}_d(\B{x})\B{v}_d(\B{x})^T d\mu(\B{x})C^T\\
		&=CM_d(\mu)C^T,
	\end{align*}
	which leads to $N_{d}(\mu)^{-1} = C^{-T} M_d(\mu)^{-1} C^{-1}$. Using Definition 1, for all $\B{x} \in \RR^p$, we have
	\begin{align*}
		R_{\mu,d}(\B{x}) &= \B{w}_d(\B{x})^T N_{d}(\mu)^{-1}\B{w}_d(\B{x})\\ 
		&=\B{v}_d(\B{x})^TC^TC^{-T} M_d(\mu)^{-1} C^{-1}C \B{v}_d(\B{x})\\
		&=\B{v}_d(\B{x})^T M_d(\mu)^{-1} \B{v}_d(\B{x})\\
		&= Q_{\mu,d}(\B{x}),
	\end{align*}
	which concludes the proof.
\end{proof}
\setcounter{lemma}{0}
\begin{lemma}
	\label{lem:affineInvarianceAppendix}
	Let $\mu$ satisfy Assumption \ref{ass:mainAssumptionAppendix} and $A \in \RR^{p\times p}, b\in \RR^p$ define an invertible affine mapping on $\RR^p$, $\mathcal{A}\colon\B{x} \to A \B{x} + b$. Then, the push foward measure, defined by $\tilde{\mu}(S) = \mu(\mathcal{A}^{-1}(S))$ for all Borel sets $S \subset \RR^p$, satisfies Assumption \ref{ass:mainAssumptionAppendix} (with the same $d$ as $\mu$) and for all $\B{x} \in \RR^p$, $Q_{\mu,d}(\B{x}) = Q_{\tilde{\mu},d}(A\B{x} + b)$.
\end{lemma}
\begin{proof}
	Let us first compute $M_{d}(\tilde{\mu})$. For the push forward measure $\tilde{\mu}$, it holds that for any $\mu$ integrable function $f \colon \RR^p \to \RR$, 
	\begin{align*}
		\int_{\RR^p} f(\B{x}) d\tilde{\mu}(\B{x}) = \int_{\RR^p} f(A\B{x} + b) d\mu(\B{x}).
	\end{align*}
	By considering polynomial $f$, we have that $\tilde{\mu}$ has all its moments finite and satisfies Assumption \ref{ass:mainAssumptionAppendix} with the same $d$ as $\mu$. Furthermore, we have
	\begin{align}
		\label{eq:affineInv1}
		M_{d}(\tilde{\mu}) &= \int_{\RR^p} \B{v}_d(\B{x})\B{v}_d(\B{x})^T d\tilde{\mu}(\B{x}) = \int_{\RR^p} \B{v}_d(A\B{x} + b)\B{v}_d(A\B{x} + b)^T d\mu(\B{x}).
	\end{align}
	We can deduce the following identity for all $\B{x} \in \RR^p$,
	\begin{align}
		\label{eq:affineInv2}
		Q_{\tilde{\mu},d}(A\B{x} + b) = \B{v}_d(A\B{x} + b)^T M_{d}(\tilde{\mu})^{-1}\,\B{v}_d(A\B{x} + b).
	\end{align}
	It remains to notice that mappings defined by $\B{w}_d(\B{x}) = \B{v}_d(A\B{x} + b)$ for all $\B{x} \in \RR^p$ form a basis of the polynomials of degree up to $d$ on $\RR^p$ (by invertibility of the affine mapping). Combining (\ref{eq:affineInv1}) and (\ref{eq:affineInv2}), we see that $\B{x} \mapsto \B{v}_d(A\B{x} + b)$ simply corresponds to the use of a different basis of $\RR_d[X]$. The result follows by applying Lemma \ref{lem:auxLemma1Appendix} and the proof is complete.
\end{proof}
\subsection{Proof of Lemma 2}
Recall that the orthogonal polynomials satisfy for all $\alpha \in \NN_d^p$
\begin{equation}
	\tag{3}
\langle P_\alpha, X^\beta \rangle = 0\ {\rm if} \ \alpha \lgl \beta,\, \langle P_\alpha,P_\alpha\rangle_\mu=1,\,\langle P_\alpha, X^\beta \rangle_\mu = 0\ {\rm if} \ \beta \lgl \alpha,
\,\langle P_\alpha, X^\alpha \rangle_\mu > 0.
\label{eq:orthoDefAppendix}
\end{equation}
\begin{lemma}
	\label{lem:connectionOrthoAppendix}
	Let $\mu$ satisfy Assumption \ref{ass:mainAssumptionAppendix}, then $Q_{\mu,d} = \sum_{\alpha \in \NN^p_d} P_\alpha^2$, where the family $\left\{P_\alpha  \right\}_{\alpha \in \NN^p_d}$ is defined by (\ref{eq:orthoDefAppendix}) and $\int_{\RR^p} Q_{\mu,d}(\B{x})d\mu(\B{x}) = s(d)$.
\end{lemma}
\begin{proof}
Let $D_d(\mu)$ be the lower triangular matrix which rows are the coefficients of the polynomials $P_\alpha$ defined in (\ref{eq:orthoDefAppendix}) ordered by $\leqgl$. From properties in (\ref{eq:orthoDefAppendix}), $D_d(\mu)$ is lower triangular with positive coefficients on its diagonal and therefore invertible. We have $D_d(\mu) M_d(\mu) D_d(\mu)^T = I$, the identity. It follows that $M_d(\mu) = D_d(\mu)^{-1} D_d(\mu)^{-T}$ and $M_d(\mu)^{-1} = D_d(\mu)^T D_d(\mu)$. Plugging this in definition 1 and using equation (1) leads to the desired identity. The average value result follows because we manipulate an orthonormal basis of $s(d)$ polymials, each of which has a square average value (with respect to $\mu$) equal to $1$.
\end{proof}
\subsection{Proof of Theorem 1}
 We recall the the optimization problem.
\begin{align}
	\label{eq:optPolyAppendix}
	\tag{4}
	\min_{Q_\alpha, \theta_\alpha, \alpha \in \NN^p_d}\;&\;\frac{1}{2}\int_{\RR^p} \sum_{\alpha \in \NN^p_d} Q_{\alpha}(\B{x})^2 d\mu(\B{x})\\
	\mathrm{s.t.}\;&\; \B{q}_{\alpha\alpha}\, \geq \exp(\theta_{\alpha}),\quad \alpha \in \NN^p_d,\nonumber\\
	&\; \B{q}_{\alpha\beta} = 0,\, \alpha \lgl \beta,\quad \alpha, \beta \in \NN^p_d,\nonumber\\
	&\;\sum_{\alpha \in \NN^p_d} \theta_\alpha= 0.\nonumber
\end{align}
where $Q_\alpha(\B{x})=\sum_\beta \B{q}_{\alpha\beta}\B{x}^\beta$, $\alpha\in\NN^p_d$. The statement of Theorem 1 goes as follows.
\setcounter{theorem}{0}
	\begin{theorem}
  \label{prop:mainRes}:
	Problem (\ref{eq:optPolyAppendix}) is a convex optimization problem with a unique optimal solution $(Q^*_\alpha,\theta_\alpha^*)$, which satisfies $Q^*_\alpha= \sqrt{\lambda}P_\alpha$, $\alpha\in\NN^p_d$, for some $\lambda>0$. In particular,
		the distinguished SOS  polynomial 
		\[Q_{\mu,d} \,=\,\sum_{\alpha\in\NN^p_d}P_\alpha^2\,=\, \frac{1}{\lambda} \sum_{\alpha \in \NN^p_d} (Q^*_{\alpha})^2,\]
		is (part of) the unique optimal solution of (\ref{eq:optPolyAppendix}).
\end{theorem}
\begin{proof}	
	\paragraph{General remarks.} Observe that (\ref{eq:optPolyAppendix}) is a convex optimization problem as we have $\int_{\RR^p} \sum_{\alpha \in \NN^p_d} Q_{\alpha}(\B{x})^2 d\mu(\B{x}) = \sum_{\alpha \in \NN^p_d} \B{q}_\alpha^T\,M_d(\mu)\,\B{q}_\alpha$, which is strictly convex in $\left\{\B{q}_\alpha  \right\}_{\alpha \in \NN^p_d}$. The proof is based on KKT optimality conditions for Problem (\ref{eq:optPolyAppendix}). We first prove that any optimal solution should be of the form $Q^*_\alpha= \sqrt{\lambda}P_\alpha$, $\alpha\in\NN^p_d$, for some $\lambda>0$. Then we show that there exists a solution of the KKT system which has this form and finally that this solution is unique. The conclusion of Theorem 1 will then follow from Lemma 1. We begin with some notations that we will use throughout the proof.

	\paragraph{Notation.} Let $\left\{ \B{e}_{\alpha} \right\}_{\alpha \in \NN^p_d}$ denote the canonical basis of $\RR^{s(d)}$ indexed by $\alpha \in \NN_d^p$ according to $\leqgl$ order. The orthonormal polynomials 
	$\left\{ P_{\alpha} \right\}_{\alpha \in \NN_d^p}$ (with respect to $\mu$)
	are uniquely defined. For each $\alpha \in \NN_d^p$, we write $\B{p}_\alpha = \left( \B{p}_{\alpha\beta} \right)_{\beta \in \NN_d^p} \in \RR^{s(d)}$ the coefficients of the polynomial $P_\alpha$. 
	By construction of $P_\alpha$, for every $\alpha,\beta \in \NN_d^p$, $\alpha \lgl \beta$, $\B{p}_{\alpha\beta} = 0$ and $\B{p}_{\alpha\alpha} > 0$.

	\paragraph{Optimality conditions}
	Problem (\ref{eq:optPolyAppendix}) is strictly feasible, we can choose 
	any $\theta$ such that $\sum_\alpha\theta_\alpha=0$ and for every $\alpha\in\NN^p_d$, set
	$Q_\alpha:=\kappa P_\alpha$ for some sufficiently large $\kappa>0$.
	Therefore the KKT optimality conditions are necessary and sufficient for global optimality. We introduce Lagrange multipliers for problem (\ref{eq:optPolyAppendix}): $\lambda_\alpha \geq 0$ for each inequality constraint, $\lambda_{\alpha \beta} \in \RR$ for each linear equality constraint on polynomials with $\alpha \lgl \beta$ and $\lambda \in \RR$ for the last linear equality constraint on $\left\{ \theta_\alpha \right\}_{\alpha \in \NN^p_d}$. The KKT optimality conditions for problem (\ref{eq:optPolyAppendix}) can be written as follows
	\begin{align}
		\lambda_{\alpha} &\geq 0,\quad \B{e}_{\alpha}^T \B{q}^*_\alpha \geq \exp(\theta^*_{\alpha}),\quad\alpha \in \NN^p_d,\label{eq:KKT1}\\
		\B{e}_{\beta}^T \B{q}^*_\alpha&= 0,\quad \alpha,\, \beta \in \NN^p_d, \quad\alpha \lgl \beta,\label{eq:KKT2}\\
		\sum_{\alpha\in \NN^p_d} \theta^*_\alpha &= 0,\label{eq:KKT3}\\
		M_d(\mu) \B{q}^*_{\alpha} &= \lambda_\alpha \B{e}_\alpha + 
		\sum_{\alpha \lgl \beta} \lambda_{\alpha\beta}\,\B{e}_{\beta},\quad \alpha \in \NN^p_d,\label{eq:KKT4}\\
		\lambda_\alpha \exp(\theta^*_{\alpha}) &= \lambda_\alpha \B{e}_{\alpha}^T\B{q}^*_{\alpha} = \lambda,\quad \alpha \in \NN^p_d,\label{eq:KKT5}
	\end{align}
	for optimal variables $\theta_{\alpha}^*$, polynomials $Q_{\alpha}^*$ with coefficients $\B{q}_{\alpha}^* \in \RR^{s(d)}$, for each $\alpha \in \NN_d^p$. We next show that the part $(Q_\alpha^*)_{\alpha \in \NN_d^p}$  of an optimal solution is necessarily	a family of orthogonal polynomials. 
	\paragraph{Any optimal solution has the form $Q^*_\alpha= \sqrt{\lambda}P_\alpha$, $\alpha\in\NN^p_d$, for some $\lambda>0$.} 
	Since KKT conditions are necessary and sufficient for optimality, we only focus on them. For each $\alpha\neq0$ and $\beta\lgl\alpha$, multiplying (\ref{eq:KKT4}) by $\B{e}_\beta$, we obtain
	\begin{align}
		\label{eq:orthoCondit1}
		\left\langle X^\beta, Q_\alpha^* \right\rangle_{\mu} &= \int \B{x}^\beta Q^*_\alpha(\B{x})\,d\mu(\B{x})\,=\,\B{e}_\beta^T\,M_d(\mu)\,\B{q}^*_\alpha\,=\,\lambda_\alpha\,\B{e}_\beta^T\B{e}_\alpha+\sum_{\alpha \lgl \gamma} \lambda_{\alpha\gamma}\,\B{e}_{\beta}^T\B{e}_\gamma\,=\,0.
	\end{align}
Similarly, multiplying (\ref{eq:KKT4}) by $\B{q}^*_\alpha$ yields for all $\alpha \in \NN_d^p$,
\begin{align}
	\label{eq:orthoCondit2}
	\left\langle Q_\alpha^*, Q_\alpha^* \right\rangle_{\mu}&=\int Q^*_\alpha(\B{x})^2\,d\mu(\B{x})\,=\,(\B{q}^*_\alpha)^T\,M_d(\mu)\,\B{q}^*_\alpha\,=\,\lambda_\alpha\,(\B{q}^*_\alpha)^T\B{e}_\alpha\,=\,\lambda,
\end{align}
where we have used (\ref{eq:KKT5}) for the last identity. In particular, with $\alpha=0$, $Q^*_0(\B{x})=\B{q}^*_{00}$ ($\geq\exp(\theta^*_0)$) for all $\B{x}$ and so
\[\lambda\,=\,\int Q^*_0(\B{x})^2\,d\mu(\B{x})\,=\,(\B{q}^*_{00})^2\,\int d\mu\geq\, \exp(2\theta^*_0),\]
which shows that $\lambda>0$.
Next, combining (\ref{eq:orthoCondit1}), (\ref{eq:orthoCondit2}) and the condition (\ref{eq:KKT2}), we immediately deduce 
\begin{align}
	\label{eq:orthoCondit3}
	\left\langle Q_\beta^*, Q_\alpha^* \right\rangle_{\mu}&=\int Q^*_\beta(\B{x})\,Q^*_\alpha(\B{x})\,d\mu(\B{x})= 
	\begin{cases}
		\lambda &\mbox{if } \alpha=\beta\\
		0& \text{otherwise}.
	\end{cases}
\end{align}
Finally, for every $\alpha \in \NN_d^p$, multiplying (\ref{eq:KKT4}) by $\B{e}_{\alpha}$ yields
\begin{align}
	\label{eq:orthoCondit4}
	\left\langle X^\alpha, Q_\alpha^* \right\rangle_{\mu}&= \int \B{x}^\alpha Q^*_\alpha(\B{x})\,d\mu(\B{x})\,=\, \B{e}_\alpha^T\,M_d(\mu)\,\B{q}^*_\alpha\,=\,\lambda_\alpha > 0,\quad \alpha\in\NN^p_d.
\end{align}
The last inequality follows from (\ref{eq:KKT5}). Indeed, suppose $\lambda_\alpha=0$ for  some $\alpha\in \NN_d^p$, this would yield $\lambda=0$. Since we have shown that $\lambda >0$, it must also hold that $\lambda_\alpha>0$ for all $\alpha$. Combining relations (\ref{eq:KKT2}), (\ref{eq:orthoCondit1}), (\ref{eq:orthoCondit3}) and (\ref{eq:orthoCondit4}), we have shown that the $\left\{Q^*_\alpha  \right\}_{\alpha \in \NN_d^p}$ form a family of orthogonal polynomials with respect to $\mu$. In addition, by the uniqueness of the orthonormal basis $\left\{P_\alpha  \right\}_{\alpha\in \NN_d^p}$, it follows from (\ref{eq:orthoCondit3}) that $Q^*_\alpha=\sqrt{\lambda}\,P_\alpha$ for every $\alpha\in\NN^p_d$. 

\paragraph{There exists a solution of this form.}
Recall that, for each $\alpha \in \NN_d^p$, $\B{p}_\alpha = \left( \B{p}_{\alpha\beta} \right)_{\beta \in \NN_d^p} \in \RR^{s(d)}$ is the vector of coefficients of the polynomial $P_{\alpha}$ which satisfies by construction $\B{p}_{\alpha\alpha} > 0$ and $\B{p}_{\alpha\beta} = 0$ for all $\beta \in \NN_d^p$, $\alpha \lgl \beta$. We use the following assignment for the primal and dual variables.
\begin{align}
	\label{eq:optAssignment}
	\lambda&=\left(\prod_{\alpha \in \NN_d^p}\B{p}_{\alpha\alpha}  \right)^{\frac{-2}{s(d)}} > 0\\
	\lambda_\alpha &= \sqrt{\lambda} \B{e}_{\alpha}^T M_d(\mu) \B{p}_\alpha = \frac{\sqrt{\lambda}}{\B{p}_{\alpha\alpha}} > 0,\quad \alpha \in \NN_d^p\nonumber\\
	\lambda_{\alpha\beta} &= \sqrt{\lambda} \B{e}_\beta^T M_d(\mu) \B{p}_\alpha,\quad \alpha,\beta \in \NN_d^p, \alpha \lgl \beta\nonumber\\
	\B{q}_\alpha^* &= \sqrt{\lambda} \B{p}_\alpha,\quad \theta_\alpha^* = \log(\sqrt{\lambda} \B{p}_{\alpha\alpha}), \quad \alpha \in \NN_d^p\nonumber.
\end{align}
Using orthonormality of the polynomials $\left\{ P_{\alpha} \right\}_{\alpha \in \NN_d^p}$, it can be check that the assignment (\ref{eq:optAssignment}) satisfies KKT optimality conditions (\ref{eq:KKT1}), (\ref{eq:KKT2}), (\ref{eq:KKT3}), (\ref{eq:KKT4}) and (\ref{eq:KKT5}). We have therefore constructed an optimal solution of (\ref{eq:optPolyAppendix}) with the desired form.
	
\paragraph{The optimal solution is unique.} From what precedes any optimal solution of (\ref{eq:optPolyAppendix})  is necessarily such that $Q^*_\alpha=\sqrt{\lambda}P_\alpha$, for every $\alpha\in\NN^n$, for some $\lambda > 0$. In addition the optimal value of (\ref{eq:optPolyAppendix}) is $s(d)\lambda$. Suppose that there exists two different optimal solutions $(Q_\alpha,\theta_\alpha)_{\alpha\in \NN_d^p}$ and
$(Q_\alpha',\theta_\alpha')_{\alpha\in \NN_d^p}$ with associated dual variables $(\lambda, \lambda_\alpha, \lambda_{\alpha\beta})_{\alpha,\beta\in \NN_d^p}$ and $(\lambda', \lambda'_\alpha, \lambda'_{\alpha\beta})_{\alpha,\beta\in \NN_d^p}$. Then necessarily $\lambda=\lambda'$, $Q_\alpha=Q_\alpha'=\sqrt{\lambda}P_\alpha$ and $\lambda_\alpha,\lambda'_\alpha >0$ for all $\alpha \in \NN_d^p$. But then from (\ref{eq:KKT5}), $\sqrt{\lambda}\B{p}_{\alpha\alpha}=\exp(\theta_\alpha)=\exp(\theta_\alpha')$ 
and so $\theta_\alpha'=\theta_\alpha$ for every $\alpha\in\NN^p_d$. Therefore the solution is unique and this concludes the proof of Theorem 1.
\end{proof}
\subsection{Proof of Theorem 2}
\begin{theorem}\label{th:christoffelAppendix}
	Let Assumption \ref{ass:mainAssumptionAppendix} hold and let $\bar{\B{x}}\in\RR^p$ be fixed, arbitrary. Then
	\begin{equation}
		\label{eq:christoffelAppendix}
		\tag{7}
		Q_{\mu,d}(\bar{\B{x}})^{-1}\,=\,\min_{P\in\RR_d[X]}\:\left\{\int P(\B{x})^2\,d\mu: P(\bar{\B{x}})\,=\,1 \right\}.
	\end{equation}
\end{theorem}
\begin{proof}
	
Fix an arbitrary $P \in \RR_d[X]$ and $\bar{\B{x}} \in \RR^p$. Assume that $P(\bar{\B{x}}) = 1$. Letting for all $\alpha \in \NN_d^p$, $a_\alpha = \left\langle P, P_{\alpha}\right\rangle_\mu$, by orthonormality, we have
\begin{align}
	\label{eq:affineOrtho}
	P = \sum_{\alpha \in \NN_d^p} a_\alpha P_\alpha,\\
	\left\langle P, P\right\rangle_\mu = \sum_{\alpha \in \NN_d^p} a_\alpha^2.\nonumber
\end{align}
The assumption that $P(\bar{\B{x}}) = 1$ can be used in conjonction with Cauchy-Schwartz inequality to obtain
\begin{align}
	\label{eq:affineOrtho2}
	1 & = P(\bar{\B{x}}) \\
	&= \sum_{\alpha \in \NN_d^p} a_{\alpha} P_\alpha(\bar{\B{x}})\nonumber\\
	&\leq \left( \sum_{\alpha \in \NN_d^p} a_{\alpha}^2 \right)\left( \sum_{\alpha \in \NN_d^p} P_\alpha(\bar{\B{x}})^2 \right)\nonumber\\
	& = \left\langle P, P\right\rangle_\mu Q_{\mu,d}(\bar{\B{x}}),\nonumber
\end{align}
where the last equality comes from the definition of $\left\langle \cdot,\cdot\right\rangle_\mu$ and Lemma 2. There is equality in equation \ref{eq:affineOrtho2} if and only if $a_{\alpha} = P_{\alpha}(\bar{\B{x}}) / Q_{\mu,d}(\bar{\B{x}})$ which always leads to $P(\bar{\B{x}}) = 1$. This shows that the infimum is attained and concludes the proof.
\end{proof}
\newpage
\section{Details about the preparation of the datasets}
We reproduce the exact same manipulations as in the references. We downloaded the \texttt{kddcup.data} from the following repository 

\url{https://archive.ics.uci.edu/ml/machine-learning-databases/kddcup99-mld/}

This file contains $4898431$ instances of network connections described by 42 features including the type of connection (attack or normal). We filter the records by keeping only those for which the variable \textit{logged in} is positive. We kept the labels (type of connection) together with the four most important features: \textit{service, duration, src\_bytes, dst\_bytes}. We applied to the three last variables (numerical) the function $\log(\cdot + 0.1) / 10$. We build four datasets with the four most frequent instances of \textit{service} and group all the remaining records in the dataset \textit{others} to get our five datasets.
\section{Illustration of affine invariance}
The following Figure illustrate the affine invariance property described in Lemma 1.
\begin{figure}[h!]
	\centering
	\includegraphics[width=\textwidth]{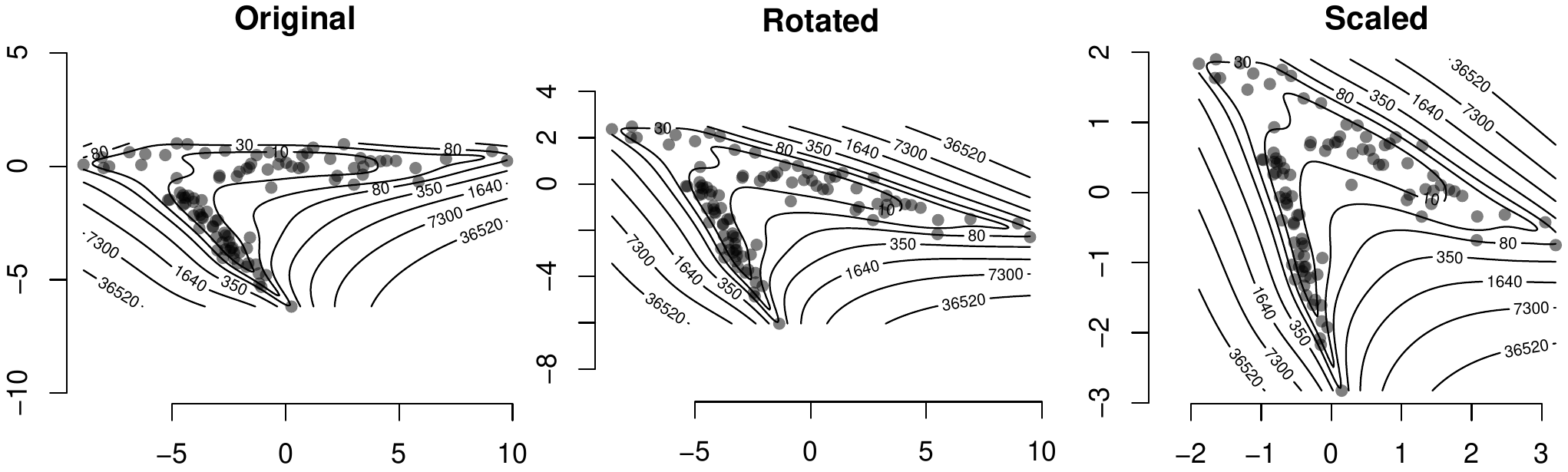}
	\caption{Empirical measure and level sets of $Q_{\mu,d}$  for three configurations of the same cloud of points ($d = 4$). The cloud in the middle is rotation of the original one and the cloud on the right is the same after centering and scaling. We observe that the level sets follow the same transformations.}
	\label{fig:examplesAffine}
\end{figure}
\end{document}